\documentclass[11pt]{amsart}
\usepackage{amsmath}

\usepackage{geometry}                % See geometry.pdf to learn the layout options. There are lots.
\usepackage{graphicx}
\usepackage{amssymb,enumerate}
\usepackage{epstopdf}
\usepackage{xcolor}
\usepackage{cancel}
\usepackage{todonotes}
\usepackage{algorithm}
\usepackage{algpseudocode}
\usepackage{subcaption}
\usepackage[numbers]{natbib}	
\usepackage{booktabs,makecell,siunitx}

\usepackage{algorithm}
\usepackage{algpseudocode}

\theoremstyle{plain}   
\newtheorem{thm}{Theorem}[section]   
\newtheorem{lem}{Lemma}[section]   
   
\newtheorem{assumption}{Assumption}[section]
\theoremstyle{definition}

\theoremstyle{remark}   
\newtheorem{rem}{Remark}[section]

\newcommand{\Prob}{\operatorname{P}}

\newcommand{\Var}{\operatorname{Var}}

\newcommand{\E}{\operatorname{E}}

\newcommand{\argmin}{\operatorname{arg\,min}}

\newcommand{\eqdis}{\stackrel{\lower0.2ex\hbox{$\scriptscriptstyle   
                    \mathrm{d}$}}{=}}

\definecolor{cadmiumgreen}{rgb}{0.0, 0.42, 0.24}                    

\title[Convergence of SARSA with linear function approximation]{Convergence of SARSA with linear function approximation: The random horizon case}
\author{Lina Palmborg}
\date{May 15, 2023}                                           % Activate to display a given date or no date
\email{lina.palmborg@math.su.se}
\thanks{Department of Mathematics, Stockholm University, Box 106 91, Stockholm, Sweden}

\begin{document}
\maketitle

\begin{abstract}
The reinforcement learning algorithm SARSA combined with linear function approximation has been shown to converge for infinite horizon discounted Markov decision problems (MDPs). 
In this paper, we investigate the convergence of the algorithm for random horizon MDPs, which has not previously been shown. We show, similar to earlier results for infinite horizon discounted MDPs, that if the behaviour policy is $\varepsilon$-soft and Lipschitz continuous with respect to the weight vector of the linear function approximation, with small enough Lipschitz constant, then the algorithm will converge with probability one when considering a random horizon MDP. 

\end{abstract}

\bigskip 

\noindent

\bigskip

\section{Introduction}
In reinforcement learning, an agent learns by interacting with an environment, and adjusts its behaviour (or policy) based on rewards received. The goal is to behave in such a way that the expected total rewards over the time horizon considered are maximised.
Central to most reinforcement learning algorithms is the estimation of values, in terms of either a state-value function representing the expected total reward of each state, or an action-value function representing the expected total reward of each action and state. 

Temporal difference (TD) learning algorithms are among the most popular reinforcement learning algorithms. In the tabular case (when the state and action set are small enough so that the values in each state or each state-action pair can be stored as tables) these type of algorithms have been shown to converge with probability one, both when estimating the value function given a specific policy \cite{dayan1992convergence}, and in TD control algorithms such as SARSA \cite{singh2000convergence} and Q-learning \cite{watkins1992q}. However, when the state space becomes large, tabular solution methods are no longer feasible. In this case, the algorithms need to be combined with function approximation when estimating the value (or action-value) function. 

However, when combining these methods with function approximation, 
convergence results have proven more difficult to obtain, 
even for algorithms using linear function approximation. In fact, there are examples of divergence of off-policy algorithms (such as Q-learning) when combined with linear function approximation in the literature, see e.g.\ \cite{tsitsiklis1997analysis}. 
Semi-gradient TD learning methods, in which the value function for a specific given policy is estimated using linear function approximation, have been shown to converge with probability one, see e.g.\ \cite{dayan1992convergence,tsitsiklis1997analysis}. 
As for SARSA combined with linear function approximation, some convergence results where obtained in the 2000s, for the case of infinite horizon discounted Markov decision problems (MDPs).  

To begin, de Farias \& Van Roy \cite{de2000existence} considered an infinite horizon discounted MDP and showed that a variant of approximate policy iteration, motivated by TD learning algorithms, using linear function approximation and a softmax behaviour policy, is guaranteed to possess at least one fixed point. The continuity of the behaviour policy with respect to the weight vector of the function approximation was needed for this result. However, convergence results were not obtained. 
Next, Gordon \cite{gordon2001reinforcement} considered a random horizon MDP, and showed that SARSA combined with linear function approximation will converge to a region, when using an $\varepsilon$-greedy behaviour policy, and that the parameters of the function approximation (and hence the derived policy) might oscillate in that region. Since an $\varepsilon$-greedy policy is discontinuous in the action-values (and thus the weight vector of the approximation), the results of \cite{gordon2001reinforcement} and \cite{perkins2003convergent} gave an indication that the continuity of the behaviour policy with respect to the weight vector of the function approximation might be of importance for providing a convergence result. 
This was further explored by Perkins \& Precup \cite{perkins2003convergent}, who considered an infinite horizon discounted MDP. They showed that a variant of SARSA combined with linear function approximation converges to a unique policy, under the condition that the behaviour policy used is $\varepsilon$-soft and Lipschitz continuous w.r.t.\ the weight vector of the function approximation with a sufficiently small Lipschitz constant. However, in the variant of SARSA that they used, the behaviour policy was only updated based on the new action-values after the weight vector had converged, i.e.\ not in an online fashion. Hence, this version of the algorithm is likely to converge slowly in practice. Melo et al.\ \cite{melo2008analysis}, who again considered an infinite horizon discounted MDP, later extended the results from \cite{perkins2003convergent} to the case when the policy is updated after each iteration, i.e.\ the standard online version of SARSA with linear function approximation commonly used in practice. 

However, to date, similar results have not been shown for random horizon MDPs. The results in \cite{de2000existence,perkins2003convergent,melo2008analysis} are obtained for infinite horizon MDPs, and all use the assumption that the Markov chain induced by any policy is irreducible and aperiodic \cite{de2000existence,perkins2003convergent} or uniformly ergodic \cite{melo2008analysis}. For a random horizon MDP, the Markov chain induced by a policy is (under certain assumptions) absorbing, hence the results of \cite{de2000existence,perkins2003convergent,melo2008analysis} are not directly applicable in this case. 

We combine results and ideas from \cite{de2000existence,perkins2003convergent,melo2008analysis}, but adjusted and applied to random horizon MDPs, together with results from \cite{bertsekas1996neuro}. The algorithm studied in this paper is SARSA with linear function approximation, but where the weight vector and the policy is updated at the end of each trajectory, after reaching the absorbing state. We show that using this algorithm to solve a random horizon MDP, when the behaviour policy is $\varepsilon$-soft and Lipschitz continuous w.r.t.\ the weight vector with sufficiently small Lipschitz constant, the weight vector will converge with probability one, hence extending the results of \cite{de2000existence,perkins2003convergent,melo2008analysis} to the random horizon MDP case.

The paper is organised as follows. 
Section \ref{sec:MDP} presents the random horizon Markov decision problem and its associated state-value function and action-value function. 
Section \ref{sec:SARSA} defines the algorithm studied in this paper. 
Section \ref{sec:Convergence} presents our main convergence result, the assumptions used, and the proof of the convergence result. 
We conclude with a discussion of the results in Section \ref{sec:Discussion}.

\section{Markov decision problem}\label{sec:MDP}
We consider a random horizon (episodic) Markov decision problem, with a finite state set $\mathcal S^+$, and a finite action set $\mathcal A$. We let $\mathcal S$ denote the set of non-terminal states, hence the set of terminal (absorbing) states is given by $\mathcal S^+\setminus \mathcal S$. If the process is in state $s$ at time $t$ and the agent chooses action $a$, the process will transition to state $s'$ with probability $p(s'\mid s,a)$. After choosing action $a$, the agent receives the reward $r(s,a)$. For the case when the reward also depends on the state $s'$ at time $t+1$, we denote the reward
by $r(s,a,s')$, and let $r(s,a)$ be the expected value of the reward:
$$
r(s,a) = \sum_{s'\in\mathcal S^+} p(s'\mid s,a) r(s,a,s'). 
$$
A policy $\pi$ can be deterministic or stochastic. A deterministic policy determines what action to take in each state, while a stochastic policy assigns a probability distribution over the set of available actions to each state $s\in\mathcal S$. The probability of choosing action $a$ in state $s$ is denoted by $\pi(a\mid s)$. The Markov decision problem consists of finding the policy $\pi$ that maximises the sum of the expected rewards received
\begin{align*}
\underset{\pi}{\text{maximise}} \E_{\pi} \bigg[ \sum_{t=0}^{T-1} r(S_t,A_t,S_{t+1}) \mid S_0=s\bigg],
\end{align*}
where $S_t$ is the state at time $t$, $A_t$ is the action taken at time $t$, $T$ denotes the terminal time, i.e. $T:=\min\{t:S_t\in\mathcal S^+\setminus\mathcal S\}$, and $\E_{\pi}[\cdot]$ denotes the expectation given that policy $\pi$ is used. The state-value function, $v_{\pi}$, and the action-value function, $q_{\pi}$, under policy $\pi$ are defined as
\begin{align*}
v_{\pi}(s) &= \E_{\pi} \bigg[ \sum_{t=0}^{T-1} r(S_t,A_t,S_{t+1}) \mid S_0=s\bigg],\\
q_{\pi}(s,a) &= \E_{\pi} \bigg[ \sum_{t=0}^{T-1} r(S_t,A_t,S_{t+1}) \mid S_0=s,A_0=a\bigg].
\end{align*}

Note that a policy $\pi$ can be seen as a matrix, with $|\mathcal S|$ rows and $|\mathcal A|$ columns, where each row sums to one. Here we view $\pi$ as a vectorised version of this matrix, i.e.\ $\pi$ is an element of $\mathbb{R}^{|\mathcal S||\mathcal A|}$. 
Let $\Delta_{\varepsilon}$ denote the set of $\varepsilon$-soft policies,
\begin{align*}
\Delta_{\varepsilon} = \Big\{\pi\in\mathbb{R}^{|\mathcal S||\mathcal A|}:\sum_{a} \pi(a\mid s) = 1 \text{ for all } s, \pi(a\mid s)\geq \varepsilon \text{ for all } (s,a)\Big\}. 
\end{align*}
Note that $\Delta_{\varepsilon}$ can be viewed as a compact subset of $\mathbb{R}^{|\mathcal S||\mathcal A|}$, since it is closed and bounded. 

A policy $\pi$ is said to be proper if the Markov chain induced by $\pi$ reaches the terminal state with probability one, irrespective of starting state, see further \cite[Def~2.1]{bertsekas1996neuro}. 

\section{SARSA with linear function approximation} \label{sec:SARSA}

The algorithm we consider is SARSA with linear function approximation. Hence the action-value function is approximated by a parameterised function $\hat q(\cdot;\theta)$ which is a linear function of the weight vector $\theta\in\mathbb{R}^d$:
\begin{align*}
\hat q(s,a;\theta) = \phi(s,a)^\top\theta,
\end{align*}
where $\phi(s,a)$ are basis functions. We let $\Phi\in\mathbb{R}^{|\mathcal S||\mathcal A|\times d}$ denote the matrix whose $(s,a)$th row is $\phi(s,a)^\top$. 

Suppose that $(S_t)$ and $(A_t)$ are sampled trajectories of states and actions, according to some behaviour policy $\pi$. The update equation for the weight vector using SARSA with linear function approximation is then
\begin{align*}
\theta_{t+1} = \theta_t + \alpha_{t+1}\phi(S_t,A_t)(r(S_t,A_t,S_{t+1})+\phi(S_{t+1},A_{t+1})^\top\theta_t - \phi(S_t,A_t)^\top\theta_t),
\end{align*}
with the convention that $\phi(S_t,\cdot) = 0$ when $S_t\in\mathcal S^+\setminus \mathcal S$, where $\alpha_t$ is the step-size parameter. We consider a slightly modified version of SARSA, where the weight vector is only updated at the end of each trajectory, when we have reached the terminal state:
\begin{align} \label{eq:update}
\theta_{t+1} = \theta_t + \alpha_{t+1}\sum_{u=0}^{T^{(t)}-1}\phi_u^{(t+1)}(r_u^{(t+1)}+(\phi_{u+1}^{(t+1)})^\top\theta_t - (\phi_u^{(t+1)})^\top\theta_t),
\end{align}
where $\phi_u^{(t)}=\phi(S_u^{(t)},A_u^{(t)})$, $r_u^{(t)}=r(S_u^{(t)},A_u^{(t)},S_{u+1}^{(t)})$, 
and $(S_u^{(t)})_{u=0}^{T^{(t)}}$ and $(A_u^{(t)})_{u=0}^{T^{(t)}-1}$ are the sampled states and actions during trajectory $t$, and $T^{(t)}$ is the time the terminal state is reached during trajectory $t$. 
Let $X_t=(S_0^{(t)},A_0^{(t)},S_1^{(t)},A_1^{(t)},\ldots,A_{T^{(t)}-1}^{(t)},S_{T^{(t)}}^{(t)})$ denote the $t$th sampled trajectory. Then \eqref{eq:update} can be written as
\begin{align*}
\theta_{t+1} = \theta_t + \alpha_{t+1}H(\theta_t,X_{t+1}),
\end{align*}
where 
\begin{align} \label{eq:H}
H(\theta_t,X_{t+1}) = \sum_{u=0}^{T^{(t+1)}-1}\phi_u^{(t+1)}(r_u^{(t+1)}+(\phi_{u+1}^{(t+1)})^\top\theta_t - (\phi_u^{(t+1)})^\top\theta_t). 
\end{align}
We further assume that the behaviour policy generating actions is updated at the end of each trajectory, and is dependent on the weight vector $\theta$. Hence the policy generating actions during trajectory $t$ will be denoted $\pi_{\theta_t}$. 
This algorithm corresponds to Algorithm \ref{Algo:A1} below. 

\begin{algorithm} 
\small
\caption{}
\begin{algorithmic}%[1]
\State Input: $\theta$-dependent policy $\pi_{\theta}$
\State Algorithm parameters: step size parameters $(\alpha_t)$
\State Initialise $\theta_0\in\mathbb{R}^d$ arbitrarily\\
\State $\pi_0=\pi_{\theta_0}$ 
\Repeat { for $t = 0,1,2,\ldots$}

	\For {$u=0,1,2,\ldots$}
		\State Simulate/observe state $S_u$
		\If {$S_u\in\mathcal S$}
			\State Choose action $A_u\sim\pi_t(\cdot|S_u)$
		\Else
			\State $T=u$
			
			\State \textbf{break}
		\EndIf
	
	\EndFor
	\State $\theta_{t+1} = \theta_t + \alpha_{t+1}\sum_{u=0}^{T-1}\phi(S_u,A_u)(r(S_u,A_u,S_{u+1}) + \phi(S_{u+1},A_{u+1})^\top \theta_t- \phi(S_u,A_u)^\top \theta_t )$
	
	\State $\pi_{t+1} =\pi_{\theta_{t+1}}$
	\Until{approximate convergence of $(\theta_t)$}
\end{algorithmic}
\label{Algo:A1}
\end{algorithm}

\section{Convergence of the algorithm}\label{sec:Convergence}

\subsection{Preliminaries}
We make the following assumptions: 
\begin{assumption} \label{Assumption:rewards}
$|r(s,a,s')|\leq r_{\max}<\infty$. 
\end{assumption}
\begin{assumption} \label{Assumption:Phi}
(i) The columns of $\Phi$ are linearly independent, (ii) $\lVert \Phi\rVert_{\infty} = \Phi_{\max} <\infty$. 
\end{assumption}
\begin{assumption} \label{Assumption:step_size}
The step-size parameters satisfy
$\sum_t \alpha_t = \infty$, $\sum_{t} \alpha_t^2<\infty$.
\end{assumption}
\begin{assumption} \label{Assumption:states}
All states in $\mathcal S$ are reachable with a positive probability, i.e.\ $\Prob_{\pi}(S_t=s)>0$ for some $t$, for all $s$ and $\pi$, where $\Prob_{\pi}(\cdot)$ denotes the probability given that policy $\pi$ is used. 
\end{assumption}
\begin{assumption} \label{Assumption:proper}
All policies are proper.
\end{assumption}

\begin{rem}\label{rem:proper}
That all policies are proper means that the Markov chain induced by any policy $\pi$ will reach the terminal state with probability one, irrespective of starting state. Naturally, this might not hold for all random horizon Markov decision problems. However, if one were to consider the discounted version of the problem, i.e.\
\begin{align*}
\underset{\pi}{\text{maximise}} \E_{\pi} \bigg[ \sum_{t=0}^{T-1} \gamma^t r(S_t,A_t,S_{t+1}) \mid S_0=s\bigg],
\end{align*}
where $\gamma<1$ is the discount factor, then discounting can be seen as a form of termination. This is due to that the problem with discounting can be seen as a problem without discounting where the state space is augmented by an additional (policy independent) absorbing state, and where the probability of reaching this new absorbing state is $1-\gamma$ from any transient state. See further \cite[Ch.~5.3]{puterman2005markov}. Based on this, the transition probability becomes $\widetilde p(s'\mid s,a) = \gamma p(s'\mid s,a)$. 
Hence, by considering the discounted version of the problem, and reformulating this in terms of the equivalent Markov decision problem augmented with an additional absorbing state, one can ensure that all policies are proper, even if the original Markov decision problem does not have this property. 
\end{rem}
Under Assumption \ref{Assumption:proper}, the Markov chain induced by any policy $\pi$ has $|\mathcal S|$ transient states, and $|\mathcal S^+\setminus \mathcal S|$ absorbing (terminal) states. We now consider the Markov chain over state-action pairs induced by a policy $\pi$. Since no action is taken in an absorbing state, we augment the set of actions with an additional action (or "no-action") $a^+$, which corresponds to "take no action". 
Under Assumption \ref{Assumption:proper} $\{(s,a^+): s\in\mathcal S^+\setminus \mathcal S\}$ is the set of absorbing states of this Markov chain over state-action pairs, and $\{(s,a): s\in\mathcal S, a\in\mathcal A\}$ is the set of transient states.  Let $P_{\pi}$ denote the $|\mathcal S||\mathcal A| \times |\mathcal S||\mathcal A|$ transition matrix of the Markov chain over state-action pairs induced by policy $\pi$ corresponding to transitions between transient states, i.e.\ the element in the $(s,a)$th row and $(s',a')$th column of $P_{\pi}$ is
\begin{align*}
p_{\pi}(s',a'\mid s,a) &= \Prob_{\pi}(S_t=s',A_t=a'\mid S_{t-1}=s,A_{t-1}=a)
=p(s'\mid s,a)\pi(a'\mid s'), 
\end{align*}
for $s,s'\in\mathcal S$ and $a,a'\in\mathcal A$. 
From \cite[Prop.~A.3]{puterman2005markov} we know that $(I-P_{\pi})^{-1}$ exists and satisfies
\begin{align*}
(I-P_{\pi})^{-1}=\sum_{k=0}^\infty P_{\pi}^k. 
\end{align*}
Furthermore, let $\lambda(s) =\Prob(S_0=s)$, and let $\eta_{\pi}$ denote the length $|\mathcal S||\mathcal A|$ vector whose $(s,a)$th element is the expected number of visits to state-action pair $(s,a)$ before absorption. Then
\begin{align} \label{eq:eta}
\eta_{\pi}(s,a) &:= \E_{\pi}\Big[\sum_{t=0}^\infty\mathbf{1}_{\{S_t=s,A_t=a\}}\Big] = \pi(a\mid s)\lambda(s) + \sum_{t=1}^\infty \Prob_{\pi}(S_t=s,A_t=a)\nonumber\\
&=\pi(a\mid s)\lambda(s) + \sum_{t=1}^\infty\sum_{s'\in\mathcal S, a'\in \mathcal A}p_{\pi}(s,a\mid s',a')\Prob_{\pi}(S_{t-1}=s',A_{t-1}=a')\nonumber\\
&=\pi(a\mid s)\lambda(s) + \sum_{s'\in\mathcal S, a'\in \mathcal A}p_{\pi}(s,a\mid s',a')\sum_{t=0}^\infty \Prob_{\pi}(S_{t}=s',A_{t}=a')\nonumber\\
&=\pi(a\mid s)\lambda(s) + \sum_{s'\in\mathcal S, a'\in \mathcal A}p_{\pi}(s,a\mid s',a')\eta_{\pi}(s',a'),
\end{align}
or, in matrix form, $\eta_{\pi}^\top = \lambda_{\pi}^\top + \eta_{\pi}^\top P_{\pi}$, hence $\eta_{\pi}^\top = \lambda_{\pi}^\top(I-P_{\pi})^{-1}$, where $\lambda_{\pi}$ is the length $|\mathcal S||\mathcal A|$ vector whose $(s,a)$th element is $\Prob_{\pi}(S_0 = s,A_0=a) = \pi(a\mid s)\lambda(s)$.
Let $D_{\pi}$ be the diagonal matrix whose diagonal is $\eta_{\pi}$, and let $r$ denote the length $|\mathcal S||\mathcal A|$ vector whose $(s,a)$th element is $r(s,a)$. 

\begin{assumption} \label{Assumption:contraction} 
At least one of the following statements holds: (i) $\lambda_{\pi}(s,a)>0$ for all $s,a$, (ii) $\sum_{s',a'}p_{\pi}(s',a'\mid s,a)<1$ for all $s,a$. 
\end{assumption}

\begin{rem}
If $\pi\in\Delta_{\varepsilon}$, then Assumption \ref{Assumption:contraction}(i) holds if $\lambda(s)=P(S_0=s) >0$ for all $s$. Furthermore, since $\sum_{s',a'}p_{\pi}(s',a'\mid s,a) = \sum_{s'} p(s'\mid s,a)\sum_{a'}\pi(a'\mid s')=\sum_{s'} p(s'\mid s,a)$, Assumption \ref{Assumption:contraction}(ii) holds if $\sum_{s'} p(s'\mid s,a)<1$ for all $s,a$, irrespective of which policy is used. Note that, similarly to Remark \ref{rem:proper}, $\sum_{s'} p(s'\mid s,a)<1$ for all $s,a$ can be achieved by instead considering the discounted version of the problem, since the augmented probabilities are then $\widetilde p(s'\mid s,a)=\gamma p(s'\mid s,a)$, hence $\sum_{s'} \widetilde p(s'\mid s,a)=\sum_{s'} \gamma p(s'\mid s,a)\leq \gamma <1$ for all $s,a$. 
\end{rem}

We make the following assumptions regarding the behaviour policy $\pi_{\theta}$:
\begin{assumption} \label{Assumption:policy}
(i) $\pi_{\theta}$ is Lipschitz continuous with respect to $\theta$, i.e.\ there exists a constant $C$ such that $\lVert \pi_{\theta} - \pi_{\theta'}\rVert \leq C \lVert \theta - \theta'\rVert$, (ii) $\pi_{\theta}$ is $\varepsilon$-soft, i.e.\ $\pi_{\theta}(a\mid s)\geq \varepsilon$ for all $s$ and $a$, for some $\varepsilon>0$. 
\end{assumption}
Note that the set of behaviour policies considered here can be viewed as a subset of $\Delta_{\varepsilon}$.
The norm denoted by $\lVert \cdot \rVert$ here corresponds to the Euclidean norm if applied to vectors, and to the spectral norm if applied to matrices. For further details on different norms and norm inequalities used throughout the paper, see Appendix \ref{app:norms}.

\subsection{Convergence theorem}

We begin by stating our main result, Theorem \ref{thm:convergence}, and then briefly go through an overview of the proof. The full proof can be found in Section \ref{sec:proof}, using results from Section \ref{sec:prel_results}. 

\begin{thm} \label{thm:convergence}
Assume that the assumptions listed in the previous section are satisfied. Then, 
for any $\varepsilon>0$, there exists $C_0>0$ such that, if $C<C_0$, the sequence $(\theta_t)$ generated by Algorithm \ref{Algo:A1} converges with probability one. 
\end{thm}

\noindent\textbf{Proof overview:}
To prove Theorem \ref{thm:convergence}, we want to use Theorem 1, Section 5.1 in Benveniste et al.\ \cite{benveniste2012adaptive} (restated in Theorem \ref{thm:benveniste}, Section \ref{sec:benveniste}, for completeness). Hence, we need to show that the Robbins-Monro assumption \eqref{eq:robbins}, the square integrability condition \eqref{eq:moment2} and the stability condition \eqref{eq:stability2}, are satisfied. 

The Robbins-Monro assumption clearly holds for the algorithm considered, see Section \ref{sec:robbins} below. 

That the stability condition \eqref{eq:stability2} holds is shown using similar ideas to Melo et al.\ \cite{melo2008analysis}. We begin by showing that there exists $\theta^*$ such that $A_{\pi_{\theta^*}}\theta^* + b_{\pi_{\theta^*}}=0$, where $A_{\pi_{\theta}}$ and $b_{\pi_{\theta}}$ are given by \eqref{eq:A_b}. To this end, we use results from de Farias \& Van Roy \cite{de2000existence}, but adapted to the current situation, where the MDP has terminal states. See further Lemmas \ref{lem:eta_cont_pi}-\ref{lem:F_pi_theta_fixed_point}, Section \ref{sec:prel_results}. In the next step, we show that $A_{\pi_{\theta}}$ and $b_{\pi_{\theta}}$ are Lipschitz continuous in $\theta$, based on results from Perkins \& Precup \cite{perkins2003convergent}, adapted to an MDP with terminal states. This is done in Lemmas \ref{lem:P_pi}-\ref{lem:Lipschitz_A_b}, Section \ref{sec:prel_results}. Furthermore, we note that $A_{\pi_{\theta}}$ is negative definite (Lemma \ref{lem:neg_def}), as shown in Bertsekas \& Tsitsiklis \cite{bertsekas1996neuro}. Using these three results it is possible to show that the stability condition \eqref{eq:stability2} is satisfied if the Lipschitz constant $C$ in Assumption \ref{Assumption:policy}(i) is sufficiently small. This is done in Section \ref{sec:stability}. 

That the square integrability condition \eqref{eq:moment2} holds is shown by rewriting \eqref{eq:H} in a way suggested by Sutton \cite{sutton1988learning}, as described in Gordon \cite{gordon2001reinforcement}. Then, using Assumptions \ref{Assumption:rewards} and \ref{Assumption:Phi}, together with general expressions for the expectation and variance of the number of steps before being absorbed in an absorbing Markov chain, it is shown that the square integrability condition holds (see Section \ref{sec:square_int}).\\

\subsection{Theorem from Benveniste et al.} \label{sec:benveniste}
The algorithm studied is of the form
\begin{align*}
\theta_{t+1} = \theta_t + \alpha_{t+1}H(\theta_t,X_{t+1}). 
\end{align*}
Robbins-Monro assumption: For any positive Borel function $g$,
\begin{align} \label{eq:robbins}
\E[g(\theta_t,X_{t+1})\mid\mathcal F_{t}] = \sum_{x}g(\theta_t,x)p_{\theta_t}(x),
\end{align}
where $p_{\theta}(x)=\Prob_{\theta}(X_t=x)$, where $\Prob_{\theta}(\cdot)$ denotes the probability given parameter $\theta$, and $\mathcal F_t$ is the $\sigma$-field generated by $X_t,X_{t-1},\ldots,X_1,\theta_t,\theta_{t-1},\ldots,\theta_0$, i.e.\ the conditional distribution of $X_{t+1}$, knowing the past, depends only on $\theta_t$.

Square integrability condition: 
\begin{align} \label{eq:moment2}
\text{For any } \theta, \text{ there exists } K \text{ s.t. } \E_{\theta}[\lVert H(\theta,X_{t})\rVert^2]\leq K(1+\lVert \theta\rVert^2)
\end{align}
where $\E_{\theta}[\cdot]$ denotes the expectation given parameter $\theta$. 

Stability condition: 
\begin{align} \label{eq:stability2}
\text{For any } \nu>0, \text{ there exists } \theta^* \text{ s.t. } \sup_{\nu\leq\lVert \theta-\theta^*\rVert\leq \frac{1}{\nu}}(\theta-\theta^*)^{\top} \E_{\theta}[H(\theta,X_{t+1})]< 0.
\end{align} 
\begin{thm}[Theorem 1, Section 5.1 in Benveniste et al.\ \cite{benveniste2012adaptive}]\label{thm:benveniste}
Under the assumptions above, if the sequence of step-size parameters satisfies $\sum_{t}\alpha_t=\infty$, $\sum_{t} \alpha_t^2<\infty$, then the sequence $(\theta_t)_{t\geq 0}$ converges almost surely to $\theta^*$ satisfying \eqref{eq:stability2}. 
\end{thm}

\subsection{Preliminary results} \label{sec:prel_results}
We begin by showing that there exists $\theta^*$ such that $A_{\pi_{\theta^*}}\theta^*+b_{\pi_{\theta^*}} = 0$, where 
\begin{align} \label{eq:A_b}
A_{\pi_{\theta}} = \Phi^\top D_{\pi_{\theta}}(P_{\pi_{\theta}}-I)\Phi, \quad b_{\pi_{\theta}}=\Phi^\top D_{\pi_{\theta}} r. 
\end{align}
This is done by using results from de Farias \& Van Roy \cite{de2000existence} and Bertsekas \& Tsitsiklis \cite{bertsekas1996neuro}. The proofs of Lemmas \ref{lem:fixed_point_F}-\ref{lem:F_pi_theta_fixed_point} follow directly from the proofs of corresponding Lemmas 5.3-5.5 in \cite{de2000existence}. These proofs are included in Appendix \ref{app:profs} for completeness. Lemmas \ref{lem:eta_cont_pi}-\ref{lem:theta_pi} and \ref{lem:theta_pi_cont} require new proofs to take into account that we consider an MDP with absorbing states.  
These proofs are included below. Lemma \ref{lem:neg_def} is used in the proof of Lemma \ref{lem:theta_pi_cont}, and later on used to show that the stability condition \eqref{eq:stability2} is satisfied. The proof of Lemma \ref{lem:neg_def} is identical to the last part of the proof of Lemma 6.10 in Bertsekas \& Tsitsiklis \cite{bertsekas1996neuro}, and is included in Appendix \ref{app:profs} for completeness. 

In order to show the above, we study the operators $H_{\pi}$ and $H_{\pi_{\theta}}$, defined by 
\begin{align*}
H_{\pi}\Phi\theta = \underset{\bar q \in \{\Phi\theta:\theta\in\mathbb{R}^{d}\}}{\argmin} \lVert r+P_{\pi}\Phi\theta -\bar q\rVert_{\eta_{\pi}}, \quad 
H_{\pi_{\theta}}\Phi\theta = \underset{\bar q \in \{\Phi\theta:\theta\in\mathbb{R}^{d}\}}{\argmin} \lVert r+P_{\pi_{\theta}}\Phi\theta -\bar q\rVert_{\eta_{\pi_{\theta}}},
\end{align*}
where 
\begin{align*}
\lVert q\rVert_{\eta_{\pi}} = \bigg(\sum_{s\in\mathcal S,a\in\mathcal A}\eta_{\pi}(s,a)q(s,a)^2\bigg)^{1/2} = \big(q^\top D_{\pi}q\big)^{1/2},
\end{align*}
for any $|\mathcal S||\mathcal A|$-dimensional vector $q$. 
We can also write $H_{\pi}=\Pi_{\pi}T_{\pi}$ and $H_{\pi_{\theta}}=\Pi_{\pi_{\theta}}T_{\pi_{\theta}}$, where
\begin{align*}
T_{\pi}\Phi\theta = r+P_{\pi}\Phi\theta,\quad \Pi_{\pi} q = \underset{\bar q \in \{\Phi\theta:\theta\in\mathbb{R}^{d}\}}{\argmin}\lVert q-\bar q\rVert_{\eta_{\pi}}. 
\end{align*}
Note that by Assumption \ref{Assumption:states} and that $\pi$ is $\varepsilon$-soft, all transient state-action pairs have a positive probability of being visited, hence $D_{\pi}$ is positive definite. Thus, using Assumption \ref{Assumption:Phi}(i) $\Phi^\top D_{\pi}\Phi$ is positive definite. Hence, the projection operator $\Pi_{\pi}$ is given by
\begin{align} \label{eq:projection}
\Pi_{\pi} = \Phi(\Phi^\top D_{\pi}\Phi)^{-1}\Phi^\top D_{\pi}.
\end{align}

The aim is to show that $H_{\pi_{\theta}}$ has a fixed point, i.e.\ that there exist $\Phi\theta$ such that $\Phi\theta = H_{\pi_{\theta}}\Phi\theta$. The reason for this is that any solution $\bar q=\Phi\theta^*$ to $\Phi\theta^*=H_{\pi_{\theta}}\Phi\theta$ must satisfy
\begin{align*}
\Phi^\top D_{\pi_{\theta}}(r+P_{\pi_{\theta}}\Phi\theta-\Phi\theta^*) = 0,
\end{align*}
hence the fixed point of $H_{\pi_{\theta}}$, if it exists, must satisfy
\begin{align*}
\Phi^\top D_{\pi_{\theta}}(r+P_{\pi_{\theta}}\Phi\theta-\Phi\theta) = 0,
\end{align*}
which can also be written as $A_{\pi_{\theta}}\theta+b_{\pi_{\theta}}=0$. 
Hence, if $H_{\pi_{\theta}}$ has a fixed point, then there exists $\theta^*$ such that $A_{\pi_{\theta^*}}\theta^*+b_{\pi_{\theta^*}}=0$.

\begin{lem} \label{lem:eta_cont_pi}
The expected number of visits to each state-action pair before absorption, $\eta_{\pi}$, is a continuous function of $\pi$. 
\end{lem}
\begin{proof}
By \eqref{eq:eta} we have $\eta_{\pi}^\top=\lambda_{\pi}^{\top}(I-P_{\pi})^{-1}$. $\lambda_{\pi}$ and $P_{\pi}$ are continuous functions of $\pi$, $I-P_{\pi}$ is nonsingular (see e.g.\ \cite[Prop.~A.3]{puterman2005markov}), and the matrix inverse function of a nonsingular matrix is continuous (see e.g.\ \cite[Prop.~C.5]{puterman2005markov}). Hence $\eta_{\pi}$ is a continuous function of $\pi$. 
\end{proof}

\begin{lem} \label{lem:theta_pi}
For each policy $\pi$, $H_{\pi}$ is a contraction, and there exists a unique vector $\theta_{\pi}$ such that $\Phi\theta_{\pi}=H_{\pi}\Phi\theta_{\pi}$.
\end{lem}
\begin{proof}
We begin by showing that $T_{\pi}$ is a contraction, i.e.\ that there exists $\beta\in[0,1)$ such that $\lVert T_{\pi}q-T_{\pi}q'\rVert_{\eta_{\pi}}\leq \beta \rVert q-q'\lVert_{\eta_{\pi}}$, where $q,q'\in\mathbb{R}^{|\mathcal S|\mathcal A|}$. This will hold if there exists $\beta\in[0,1)$ such that $\lVert P_{\pi} q\rVert_{\eta_{\pi}}\leq\beta\lVert q\rVert_{\eta_{\pi}}$, since 
$T_{\pi} q-T_{\pi}q' = P_{\pi} (q-q')$. 
Let $p_{\text{sum}}$ be the $|\mathcal S||\mathcal A|$-dimensional vector corresponding to the row sums of $P_{\pi}$, i.e.\ $p_{\text{sum}}=P_{\pi}\mathbf{1}$, and $p_{\text{sum}}(s,a)$ the $(s,a)$th element of $p_{\text{sum}}$. 
For $q\in\mathbb{R}^{|\mathcal S||\mathcal A|}$,
\begin{align*}
\lVert P_{\pi}q\rVert_{\eta_{\pi}}^2 &= q^\top P_{\pi}^\top D_{\pi}P_{\pi}q=\sum_{s,a}\eta_{\pi}(s,a)\Big(\sum_{s',a'}p_{\pi}(s',a'\mid s,a)q(s',a')\Big)^2\\
&= \sum_{(s,a):p_{\text{sum}}(s,a)>0}\eta_{\pi}(s,a) p_{\text{sum}}(s,a)^2 \Big(\sum_{s',a'}\frac{p_{\pi}(s',a'\mid s,a)}{p_{\text{sum}}(s,a)}q(s',a')\Big)^2\\
&\leq \sum_{(s,a):p_{\text{sum}}(s,a)>0}\eta_{\pi}(s,a) p_{\text{sum}}(s,a)^2 \sum_{s',a'}\frac{p_{\pi}(s',a'\mid s,a)}{p_{\text{sum}}(s,a)}q(s',a')^2\\
& = \sum_{s,a}\eta_{\pi}(s,a) p_{\text{sum}}(s,a) \sum_{s',a'}p_{\pi}(s',a'\mid s,a)q(s',a')^2.
\end{align*}
Under Assumption \ref{Assumption:contraction}(i), since $\eta_{\pi}(s,a)-\lambda_{\pi}(s,a)< \eta_{\pi}(s,a)$ and $p_{\text{sum}}(s,a)\leq 1$ for all $s,a$, 
\begin{align*}
\lVert P_{\pi}q\rVert_{\eta_{\pi}}^2 &\leq \sum_{s,a}\eta_{\pi}(s,a)\sum_{s',a'}p_{\pi}(s',a'\mid s,a)q(s',a')^2\\
&= \sum_{s',a'}q(s',a')^2\sum_{s,a}\eta_{\pi}(s,a)p_{\pi}(s',a'\mid s,a)\\
&=\sum_{s',a'}q(s',a')^2(\eta_{\pi}(s',a')-\lambda_{\pi}(s',a'))\\
&\leq\beta \sum_{s',a'}q(s',a')^2\eta_{\pi}(s',a') = \beta\lVert q\rVert_{\eta_{\pi}}^2, 
\end{align*}
where
\begin{align*}
\beta = 1-\min_{s,a}\frac{\lambda_{\pi}(s,a)}{\eta_{\pi}(s,a)}
\end{align*}
i.e.\ $\beta\in[0,1)$ 
since $\lambda_{\pi}(s,a)>0$ and $\eta_{\pi}(s,a)\geq \lambda_{\pi}(s,a)$ for all $s,a$.
Under Assumption \ref{Assumption:contraction}(ii), since $p_{\text{sum}}(s,a)<1$ and $\eta_{\pi}(s,a)-\lambda_{\pi}(s,a)\leq \eta_{\pi}(s,a)$ for all $s,a$,
\begin{align*}
\lVert P_{\pi}q\rVert_{\eta_{\pi}}^2 &\leq \sum_{s,a}\eta_{\pi}(s,a)p_{\text{sum}}(s,a)\sum_{s',a'}p_{\pi}(s',a'\mid s,a)q(s',a')^2\\
&\leq \max_{s,a}p_{\text{sum}}(s,a)\sum_{s',a'}q(s',a')^2\sum_{s,a}\eta_{\pi}(s,a)p_{\pi}(s',a'\mid s,a)\\
&\leq \beta \sum_{s',a'}q(s',a')^2\eta_{\pi}(s',a') = \beta\lVert q\rVert_{\eta_{\pi}}^2,
\end{align*}
where $\beta = \max_{s,a}p_{\text{sum}}(s,a)<1$. Hence, $T_{\pi}$ is a contraction under Assumption \ref{Assumption:contraction}.

Furthermore, $\Pi_{\pi}$ is a non-expansion in the sense that $\lVert \Pi_{\pi} q\rVert_{\eta_{\pi}}\leq \lVert q\rVert_{\eta_{\pi}}$ (this is a generic property of projections, and can also easily be shown based on \eqref{eq:projection}, see e.g.\ the proof of Proposition 6.9 in \cite{bertsekas1996neuro}). Hence, $H_{\pi}$ is a contraction. By the contraction mapping theorem, it follows that $H_{\pi}$ has a unique fixed point $\Phi\theta_{\pi}$. 
By Assumption \ref{Assumption:Phi}(i), we can conclude that $\theta_{\pi}$ is unique. 
\end{proof}

\begin{lem} \label{lem:neg_def}
For any $\varepsilon$-soft policy $\pi$, the matrix $A_{\pi}=\Phi^\top D_{\pi}(P_{\pi}-I)\Phi$ is negative definite.
\end{lem}

\begin{lem} \label{lem:theta_pi_cont}
Let $\theta_{\pi}$ be the unique solution to 
$\Phi\theta_{\pi}=H_{\pi}\Phi \theta_{\pi}$. For any $\varepsilon>0$, the function $\Delta_{\varepsilon}\ni\pi\mapsto\theta_{\pi}$ is continuous.  
\end{lem}
\begin{proof} 
By Assumption \ref{Assumption:states} and that $\pi$ is $\varepsilon$-soft, the projection operator $\Pi_{\pi}$ is given by \eqref{eq:projection},
and the solution $\theta_{\pi}$ to $\Phi\theta_{\pi}=H_{\pi}\Phi \theta_{\pi}$ must satisfy
\begin{align*}
\Phi\theta_{\pi} = \Phi(\Phi^\top D_{\pi}\Phi)^{-1}\Phi^\top D_{\pi}(r+P_{\pi}\Phi\theta_{\pi}).
\end{align*}
Hence, 
\begin{align*}
\Phi^\top D_{\pi}(I-P_{\pi})\Phi \theta_{\pi} = \Phi^\top D_{\pi} r.
\end{align*}

Let $A_{\pi} = \Phi^\top D_{\pi}(P_{\pi}-I)\Phi$ and $b_{\pi} = \Phi^\top D_{\pi} r$, so that $A_{\pi}\theta_{\pi} + b_{\pi}=0$. By Lemma \ref{lem:neg_def} $A_{\pi}$ is negative definite, 
and thus $\theta_{\pi} = -A_{\pi}^{-1}b_{\pi}$. $P_{\pi}$ is a continuous function of $\pi$, by Lemma \ref{lem:eta_cont_pi}, $D_{\pi}$ is a continuous function of $\pi$, and the matrix inverse function of a nonsingular matrix is continuous (see e.g.\ \cite[Prop.~C.5]{puterman2005markov}), hence both $A_{\pi}^{-1}$ and $b_{\pi}$ are continuous functions of $\pi$, for any $\varepsilon$-soft policy $\pi$. 
\end{proof}

Similarly to de Farias and Van Roy \cite{de2000existence}, but for any $\varepsilon$-soft policy $\pi$, and for any policy $\pi_{\theta}$ satisfying Assumption \ref{Assumption:policy}, we define
\begin{align*}
s_{\pi}(\theta)=\Phi^\top D_{\pi}(T_{\pi}\Phi\theta -\Phi\theta)\quad \text{and} \quad s_{\pi_{\theta}}(\theta)=\Phi^\top D_{\pi_{\theta}}(T_{\pi_{\theta}}\Phi\theta -\Phi\theta),
\end{align*}
and functions $F_{\pi}^{\alpha}:\mathbb{R}^d\to\mathbb{R}^d$ and $F_{\pi_{\theta}}^{\alpha}:\mathbb{R}^d\to\mathbb{R}^d$ by
\begin{align*}
F_{\pi}^{\alpha}(\theta) = \theta + \alpha s_{\pi}(\theta)\quad \text{and} \quad F_{\pi_{\theta}}^{\alpha}(\theta) = \theta + \alpha s_{\pi_{\theta}}(\theta). 
\end{align*}

\begin{lem} \label{lem:fixed_point_F}
For any $\alpha>0$, $\theta$ is a fixed point of $F_{\pi}^{\alpha}$ ($F_{\pi_{\theta}}^{\alpha}$) if and only if $\Phi\theta$ is a fixed point of $H_{\pi}$ ($H_{\pi_{\theta}}$). 
\end{lem}

The following lemma is used to prove Lemma \ref{lem:F_pi_theta_fixed_point}. 

\begin{lem} \label{lem:F_quasi_contraction}
There exists $\alpha^*>0$ such that, for all $\varepsilon$-soft policies $\pi$ and any $\alpha\in(0,\alpha^*)$, there exists a scalar $\beta_{\alpha}$ such that
\begin{align*}
\lVert F_{\pi}^\alpha(\theta)-\theta_{\pi}\rVert\leq\beta_{\alpha}\lVert\theta -\theta_{\pi}\rVert.
\end{align*}
\end{lem}

\begin{lem} \label{lem:F_pi_theta_fixed_point}
For any $\alpha>0$, the function $F_{\pi_{\theta}}^{\alpha}$ possesses a fixed point. 
\end{lem}

Hence, using Lemma \ref{lem:F_pi_theta_fixed_point} and Lemma \ref{lem:fixed_point_F} we can show the desired result, i.e.\ that $H_{\pi_{\theta}}$ has a fixed point. \\

Next, we show that $A_{\pi_{\theta}}$ and $b_{\pi_{\theta}}$ are Lipschitz continuous w.r.t.\ $\theta$, which, since $\pi_{\theta}$ is Lipschitz continuous w.r.t.\ $\theta$, will hold if $A_{\pi}$ and $b_{\pi}$ are Lipschitz continuous w.r.t.\ $\pi$, where
\begin{align*}
A_{\pi} = \Phi^\top D_{\pi}(P_{\pi}-I)\Phi, \quad b_{\pi}=\Phi^\top D_{\pi} r. 
\end{align*}
This will be done using results from Perkins \& Precup \cite{perkins2003convergent}. The proof of Lemma \ref{lem:P_pi} follows directly from the proof of the correspoding Lemma 1 in \cite{perkins2003convergent}. This proof is included in Appendix \ref{app:profs} for completeness. Lemma \ref{lem:D_Lipschitz} requires a new proof to take into account that we consider an MDP with absorbing states, and this proof is included below. The proof of Lemma \ref{lem:Lipschitz_A_b} is similar to the proof of Lemma 3 in \cite{perkins2003convergent}, but some adjustments are required for the case of an absorbing MDP, hence this proof is also included below. 

\begin{lem} \label{lem:P_pi}
For any $\varepsilon>0$, 
there exists $C_P$ such that  
$\lVert P_{\pi_1}-P_{\pi_2}\rVert\leq C_P\lVert\pi_1-\pi_2\rVert$ for all $\pi_1,\pi_2\in\Delta_{\varepsilon}$. 
\end{lem}

\begin{lem} \label{lem:D_Lipschitz}
For any $\varepsilon > 0$, there exists $C_D$ such that $\lVert D_{\pi_1}-D_{\pi_2}\rVert \leq C_D\rVert\pi_1-\pi_2\rVert$ for all $\pi_1,\pi_2\in\Delta_{\varepsilon}$. 
\end{lem}

\begin{proof}
Note that
\begin{align*}
(\eta_{\pi_1}^\top-\eta_{\pi_2}^\top)(I-P_{\pi_1}) &= \eta_{\pi_1}^\top-\eta_{\pi_1}^\top P_{\pi_1} -\eta_{\pi_2}^\top+\eta_{\pi_2}^\top P_{\pi_1}
= \lambda_{\pi_1}^\top-\eta_{\pi_2}^\top+\eta_{\pi_2}^\top P_{\pi_1}\\
&=\lambda_{\pi_1}^\top-\lambda_{\pi_2}^\top + \eta_{\pi_2}^\top (P_{\pi_1}- P_{\pi_2}), 
\end{align*}
hence, since $(I-P_{\pi})^{-1}$ exists for all proper policies $\pi$, 
\begin{align*}
\eta_{\pi_1}^\top-\eta_{\pi_2}^\top = \big(\lambda_{\pi_1}^\top-\lambda_{\pi_2}^\top + \eta_{\pi_2}^\top (P_{\pi_1}- P_{\pi_2})\big)(I-P_{\pi_1})^{-1}. 
\end{align*}
Then
\begin{align*}
\lVert \eta_{\pi_1}-\eta_{\pi_2} \rVert &=\lVert ((I-P_{\pi_1})^{-1})^\top \big(\lambda_{\pi_1}-\lambda_{\pi_2} +  (P_{\pi_1}^\top- P_{\pi_2}^\top)\eta_{\pi_2}\big) \rVert \\ 
& \leq \lVert (I-P_{\pi_1})^{-1}\rVert (\lVert \lambda_{\pi_1}-\lambda_{\pi_2}\rVert +\lVert  (P_{\pi_1}^\top- P_{\pi_2}^\top)\eta_{\pi_2} \rVert).
\end{align*}
Now, 
\begin{align*}
\rVert\lambda_{\pi_1}-\lambda_{\pi_2}\rVert &= \sqrt{\sum_{s,a}\lvert \pi_1(a\mid s)\lambda(s)-\pi_2(a\mid s)\lambda(s)\rvert^2}=\sqrt{\sum_{s\in\mathcal S}\lambda(s)^2\sum_{a\in\mathcal A}|\pi_1(a\mid s)-\pi_2(a\mid s)|^2}\\
&\leq \sqrt{\sum_{s,a}|\pi_1(a\mid s)-\pi_2(a\mid s)|^2} = \lVert \pi_1-\pi_2 \rVert,
\end{align*}
and
\begin{align*}
\lVert  (P_{\pi_1}^\top- P_{\pi_2}^\top) \eta_{\pi_2} \rVert&
\leq  \lVert P_{\pi_1}- P_{\pi_2}\rVert\lVert \eta_{\pi_2}\rVert \leq \lVert P_{\pi_1}- P_{\pi_2}\rVert\lVert (I-P_{\pi_2})^{-1}\rVert\rVert \lambda_{\pi_2}\rVert\\
&\leq  \lVert P_{\pi_1}- P_{\pi_2}\rVert\lVert (I-P_{\pi_2})^{-1}\rVert\rVert \lambda_{\pi_2}\rVert_1
\leq C_P \lVert (I-P_{\pi_2})^{-1}\rVert\lVert \pi_1-\pi_2\rVert
\end{align*}
using $\rVert \lambda_{\pi}\rVert_1=1$ for any policy $\pi$ and Lemma \ref{lem:P_pi}. 
Hence
\begin{align*}
\lVert \eta_{\pi_1}-\eta_{\pi_2} \rVert& \leq
\lVert (I-P_{\pi_1})^{-1}\rVert \big(1+ \lVert (I-P_{\pi_2})^{-1}\rVert C_P \big) \lVert \pi_1-\pi_2\rVert\\
&\leq \zeta \big(1+ \zeta C_P \big) \lVert \pi_1-\pi_2\rVert,
\end{align*}
where $\zeta := \max_{\pi\in\Delta_{\varepsilon}}\lVert (I-P_{\pi})^{-1}\rVert$. The maximum is attained since $P_{\pi}$ is a continuous function of $\pi$, $I-P_{\pi}$ is nonsingular, the matrix inverse function of a nonsingular matrix is continuous (see e.g.\ \cite[Prop.~C.5]{puterman2005markov}), any norm is a continuous function, and $\Delta_{\varepsilon}$ can be viewed as a compact subset of $\mathbb{R}^{|\mathcal S||\mathcal A|}$. Let $C_D=\zeta \big(1+ \zeta C_P \big)$.
Then
\begin{align*}
\lVert D_{\pi_1}-D_{\pi_2}\rVert=\max_{s,a}|\eta_{\pi_1}(s,a)-\eta_{\pi_2}(s,a)| =
\lVert \eta_{\pi_1}-\eta_{\pi_2}\rVert_{\infty}\leq \lVert \eta_{\pi_1}-\eta_{\pi_2}\rVert\leq C_D\lVert \pi_1-\pi_2 \rVert.
\end{align*}
\end{proof}

\begin{lem} \label{lem:Lipschitz_A_b}
For any $\varepsilon>0$, 
there exists $C_b$ and $C_A$ such that $\lVert b_{\pi_1}-b_{\pi_2}\rVert\leq C_b\lVert\pi_1-\pi_2\rVert$ and $\lVert A_{\pi_1}-A_{\pi_2}\rVert\leq C_A\lVert\pi_1-\pi_2\rVert$ for all $\pi_1,\pi_2\in\Delta_{\varepsilon}$. 
\end{lem} 

\begin{proof} First, note that
\begin{align*}
\lVert r \rVert &= \sqrt{\sum_{s,a}\Big(\sum_{s'\in\mathcal S^+}p(s'\mid s,a) r(s,a,s')\Big)^2}\\
&\leq \sqrt{\sum_{s,a} \sum_{s'\in\mathcal S^+}p(s'\mid s,a)|r(s,a,s')|^2}\leq \sqrt{|\mathcal S||\mathcal A|}r_{\max},
\end{align*}
using Assumption \ref{Assumption:rewards}, 
and
\begin{align*}
\lVert \Phi^\top\rVert = \lVert \Phi \rVert \leq \sqrt{|\mathcal S||\mathcal A|}\lVert \Phi\rVert_{\infty}= \sqrt{|\mathcal S||\mathcal A|}\Phi_{\max},
\end{align*}
using Assumption \ref{Assumption:Phi}(ii). 
For the first claim, using Lemma \ref{lem:D_Lipschitz},
\begin{align*}
\lVert b_{\pi_1}-b_{\pi_2}\rVert&=\lVert \Phi^\top(D_{\pi_1}-D_{\pi_2})r\rVert\leq \lVert \Phi^\top\rVert\lVert D_{\pi_1}-D_{\pi_2}\rVert\lVert r\rVert 
\leq C_D |\mathcal S||\mathcal A|\Phi_{\max}r_{\max}\lVert\pi_1-\pi_2\rVert,
\end{align*}
i.e.\ $C_b =C_D|\mathcal S||\mathcal A|\Phi_{\max}r_{\max}$.  

For the second claim, we use that  $\lVert D_{\pi}\rVert\leq \max_{\pi\in\Delta_{\varepsilon}}\lVert D_{\pi} \rVert:=\xi$ for any $\pi\in\Delta_{\varepsilon}$ (By Lemma \ref{lem:eta_cont_pi} $D_{\pi}$ is a continuous function of $\pi$, any norm is a continuous function, and $\Delta_{\varepsilon}$ is compact, hence the maximum is attained), and $\lVert P_{\pi}\rVert\leq \sqrt{|\mathcal S||\mathcal A|}\lVert P_{\pi}\rVert_{\infty} \leq \sqrt{|\mathcal S||\mathcal A|}$. Hence, using Lemmas \ref{lem:P_pi} and \ref{lem:D_Lipschitz},
\begin{align*}
\lVert A_{\pi_1}-A_{\pi_2}\rVert &=\lVert \Phi^\top \big(D_{\pi_1}(P_{\pi_1}-I) - D_{\pi_2}(P_{\pi_2}-I)\big)\Phi\rVert \\
&\leq \lVert\Phi^\top\rVert \lVert D_{\pi_2}-D_{\pi_1} +D_{\pi_1}P_{\pi_1}-D_{\pi_2}P_{\pi_2} \rVert \lVert \Phi\rVert \\
& \leq |\mathcal S||\mathcal A|\Phi_{\max}^2 \lVert D_{\pi_2}-D_{\pi_1} +D_{\pi_1}(P_{\pi_1}-P_{\pi_2})+(D_{\pi_1}-D_{\pi_2})P_{\pi_2} \rVert \\
& \leq |\mathcal S||\mathcal A|\Phi_{\max}^2 \big(\lVert D_{\pi_1}-D_{\pi_2}\rVert + \lVert D_{\pi_1}\rVert\lVert P_{\pi_1}-P_{\pi_2} \rVert + \lVert D_{\pi_1}-D_{\pi_2}\rVert \lVert P_{\pi_2}\rVert\big) \\
& \leq |\mathcal S||\mathcal A|\Phi_{\max}^2 \big((1+\sqrt{|\mathcal S||\mathcal A|})\lVert D_{\pi_1}-D_{\pi_2}\rVert +\xi\lVert P_{\pi_1}-P_{\pi_2} \rVert \big)  \\
&\leq |\mathcal S||\mathcal A|\Phi_{\max}^2 \big((1+\sqrt{|\mathcal S||\mathcal A|})C_D+\xi C_{P}\big) \lVert \pi_1-\pi_2\rVert,
\end{align*}
i.e.\ $C_A=|\mathcal S||\mathcal A|\Phi_{\max}^2((1+\sqrt{|\mathcal S||\mathcal A|})C_D+\xi C_{P})$. 
\end{proof}

By Lemma \ref{lem:Lipschitz_A_b}, it is clear that under Assumption \ref{Assumption:policy}, $A_{\pi_{\theta}}$ and $b_{\pi_{\theta}}$ are Lipschitz continuous with respect to $\theta$ with Lipschitz constants $C_1 = C_AC$ and $C_2 = C_bC$ respectively.

\subsection{Proof of Theorem \ref{thm:convergence}}\label{sec:proof}
\subsubsection{Robbins-Monro assumption}\label{sec:robbins}
The Robbins-Monro assumption \eqref{eq:robbins} holds in our case by the definition of $X_{t+1} = (S_0^{(t+1)},A_0^{(t+1)},S_1^{(t+1)},A_1^{(t+1)},\ldots,A_{T^{(t+1)}-1}^{(t+1)},S_{T^{(t+1)}}^{(t+1)})$, which is sampled according to 
\begin{align*}
\Prob_{\pi_{\theta_{t}}}(A_u^{(t+1)}=a\mid S_u^{(t+1)}=s)&=\pi_{\theta_{t}}(a\mid s), \quad \text{for } u=0,\ldots,T^{(t+1)}-1,\\ 
\Prob(S_{u+1}^{(t+1)}=s'\mid S_{u}^{(t+1)}=s,A_u^{(t+1)}=a)&=p(s'\mid s,a),\quad \text{for } u=0,\ldots,T^{(t+1)}-1,\\
\Prob(S_0^{(t+1)}=s)&=\lambda(s). 
\end{align*}
Hence each trajectory/episode is independent of the previous episodes given $\theta_t$ (and independent of $(\theta_{t-1},\theta_{t-2},\ldots)$ given $\theta_t$, since only $\theta_t$ affects the behaviour policy used during trajectory $t+1$). 

\subsubsection{Square integrability condition}\label{sec:square_int}
Similarly to Gordon \cite{gordon2001reinforcement}, we use the equations on p.~25 in Sutton \cite{sutton1988learning} to write $H(\theta,X_{t+1})$ as
\begin{align*}
H(\theta,X_{t+1})
= \sum_{s,a}\sum_{s',a'} \gamma_{\theta}(s',a'\mid s,a)\phi(s,a) \big( r(s,a,s')+\phi(s',a')^\top \theta -\phi(s,a)^\top\theta\big),
\end{align*}
where $\gamma_{\theta}(s',a'\mid s,a)$ denotes the number of times the transition $(s,a)\to (s',a')$ occurs in the sequence $X_{t+1}=(S_0^{(t+1)},A_0^{(t+1)},S_1^{(t+1)},A_1^{(t+1)},\ldots,A_{T^{(t+1)}-1}^{(t+1)},S_{T^{(t+1)}}^{(t+1)})$ (for $s'\in \mathcal S^+\setminus \mathcal S$ all but one of the $\gamma_{\theta}(s',a'\mid s,a)$ are 0). Let $\delta_{\theta}(s,a,s',a')$ be defined by
\begin{align*}
\delta_{\theta}(s,a,s',a') = r(s,a,s')+\phi(s',a')^\top \theta + \phi(s,a)^\top\theta, 
\end{align*}
so that 
\begin{align*}
H(\theta,X_{t+1})
= \sum_{s,a}\sum_{s',a'} \gamma_{\theta}(s',a'\mid s,a)\phi(s,a) \delta_{\theta}(s,a,s',a'). 
\end{align*}
Then
\begin{align*}
\lVert H(\theta,X_{t+1})\rVert &\leq \sum_{s,a}\sum_{s',a'} \gamma_{\theta}(s',a'\mid s,a) |\delta_{\theta}(s,a,s',a')|\lVert \phi(s,a)\rVert\\
&\leq \Phi_{\max}\sum_{s,a}\sum_{s',a'} \gamma_{\theta}(s',a'\mid s,a) |\delta_{\theta}(s,a,s',a')|, 
\end{align*}
where we have used Assumption \ref{Assumption:Phi}(ii). 
Furthermore, using Assumptions \ref{Assumption:Phi}(ii) and \ref{Assumption:rewards},
\begin{align*}
|\delta_{\theta}(s,a,s',a')|&\leq |r(s,a,s')| + (\lVert \phi(s',a')\rVert + \lVert \phi(s,a)\rVert)\lVert \theta \rVert\leq r_{\max}+2\Phi_{\max}\lVert \theta\rVert,
\end{align*}
hence
\begin{align*}
\lVert H(\theta,X_{t+1})\rVert &\leq \Phi_{\max}(r_{\max}+2\Phi_{\max}\lVert \theta\rVert)\sum_{s,a}\sum_{s',a'} \gamma_{\theta}(s',a'\mid s,a). 
\end{align*}
By using $(a+b)^2\leq 2(a^2+b^2)$ we obtain
\begin{align*}
\lVert H(\theta,X_{t+1})\rVert^2 \leq 2\Phi_{\max}^2(r_{\max}^2+4\Phi_{\max}^2\lVert \theta\rVert^2)\Big(\sum_{s,a}\sum_{s',a'} \gamma_{\theta}(s',a'\mid s,a)\Big)^2. 
\end{align*}
Let $\Gamma_{\theta} = \sum_{s,a}\sum_{s',a'} \gamma_{\theta}(s',a'\mid s,a)$ denote the number of steps until absorption. Then 
\begin{align*}
\E_{\theta}\big[ \lVert H(\theta,X_{t+1})\rVert^2\big] \leq 2\Phi_{\max}^2(r_{\max}^2+4\Phi_{\max}^2\lVert \theta\rVert^2)\E_{\theta}\big[\Gamma_{\theta}^2\big],
\end{align*}
where $\E_{\theta}[\cdot]$ denotes the expectation given parameter $\theta$ (and thus also given that policy $\pi_{\theta}$ is used). 
The expected number of steps before being absorbed, when starting in state $(s,a)$, is given by the $(s,a)$th element of $t_{\theta}$, denoted by $t_{\theta}(s,a)$, where $t_{\theta} = N_{\theta}\mathbf{1}$, and $N_{\theta} = (I-P_{\pi_{\theta}})^{-1}$, and the variance of the number of steps before being absorbed, when starting in state $(s,a)$, is given by the $(s,a)$th element of $(2N_{\theta}-I)t_{\theta}-t_{\theta}\odot t_{\theta}$, where $\odot$ denotes the Hadamard product, see e.g.\ \cite[Thm~3.3.5]{kemeny1976finite}.  Let $((2N_{\theta}-I)t_{\theta})_{(s,a)}$ denote the $(s,a)$th element of $(2N_{\theta}-I)t_{\theta}$. Then, 
\begin{align*}
\E_{\theta}\big[\Gamma_{\theta}^2\big] &= \E_{\theta}\big[ \E_{\theta}\big[ \Gamma_{\theta}^2 \mid S_0^{(t+1)},A_0^{(t+1)}\big] \big]\\
&= \E_{\theta} \big[ \Var_{\theta}\big(\Gamma_{\theta}\mid S_0^{(t+1)},A_0^{(t+1)}\big)\big] +\E_{\theta}\big[\E_{\theta}\big[ \Gamma_{\theta}\mid S_0^{(t+1)},A_0^{(t+1)}\big]^2\big]\\
&=\E_{\theta} \big[ ((2N_{\theta}-I)t_{\theta})_{(S_0^{(t+1)},A_0^{(t+1)})} -t_{\theta}(S_0^{(t+1)},A_0^{(t+1)})^2\big] + \E_{\theta}\big[t_{\theta}(S_0^{(t+1)},A_0^{(t+1)})^2\big]\\
&= \lambda_{\pi_{\theta}}^{\top}(2N_{\theta}-I)t_{\theta}\leq \lVert \lambda_{\pi_{\theta}}\rVert \lVert (2N_{\theta}-I)t_{\theta}\rVert 
\leq \lVert \lambda_{\pi_{\theta}}\rVert_1 \lVert 2N_{\theta}-I\rVert \lVert t_{\theta}\rVert\\
&\leq (2\lVert N_{\theta}\rVert+\lVert I\rVert)\lVert N_{\theta}\rVert \lVert \mathbf{1}\rVert \leq (2\lVert (I-P_{\pi_{\theta}})^{-1}\rVert +1)\lVert (I-P_{\pi_{\theta}})^{-1}\rVert\sqrt{|\mathcal S||\mathcal A|}\\
&\leq \sqrt{|\mathcal S||\mathcal A|}(2\zeta+1)\zeta. 
\end{align*} 
To conclude, 
\begin{align*}
\E_{\theta}\big[\lVert H(\theta,X_{t+1})\rVert^2\big] \leq 2\Phi_{\max}^2\sqrt{|\mathcal S||\mathcal A|}(2\zeta+1)\zeta(r_{\max}^2+4\Phi_{\max}^2\lVert \theta\rVert^2), 
\end{align*}
i.e.\ there exists a constant $K$ such that \eqref{eq:moment2} is satisfied. 

\subsubsection{Stability condition}\label{sec:stability}
First, note that 
\begin{align*}
h(\theta)&:=\E_{\theta}[H(\theta,X_{t+1})]
= \E_{\theta}\bigg[\sum_{u=0}^{T^{(t+1)}-1}\phi_u^{(t+1)}(r_u^{(t+1)}+(\phi_{u+1}^{(t+1)})^\top\theta - (\phi_u^{(t+1)})^\top\theta) \bigg]\\
&=\Phi^\top D_{\pi_{\theta}} r + \Phi^\top D_{\pi_{\theta}}(P_{\pi_{\theta}}-I)\Phi\theta = b_{\pi_{\theta}} + A_{\pi_{\theta}}\theta. 
\end{align*}
This can be shown similarly to the proof of Proposition 6.6 in Bertsekas \& Tsitsiklis \cite{bertsekas1996neuro}. The proof is included in Appendix \ref{app:h} for completeness. 

By Lemmas \ref{lem:fixed_point_F} and \ref{lem:F_pi_theta_fixed_point}, there 
exists $\theta^*$ such that $A_{\pi_{\theta^*}} \theta^* +b_{\pi_{\theta^*}} = 0$, hence
\begin{align*}
(\theta-\theta^*)^\top h(\theta) &= (\theta-\theta^*)^\top(A_{\pi_{\theta}}\theta + b_{\pi_{\theta}}) =  (\theta-\theta^*)^\top(A_{\pi_{\theta}}\theta - A_{\pi_{\theta^*}} \theta^* + b_{\pi_{\theta}} - b_{\pi_{\theta^*}})\\
&= (\theta-\theta^*)^\top A_{\pi_{\theta}} (\theta-\theta^*)+ (\theta-\theta^*)^\top(A_{\pi_{\theta}}-A_{\pi_{\theta^*}})\theta^* + (\theta-\theta^*)^\top(b_{\pi_{\theta}}-b_{\pi_{\theta^*}})\\
&\leq (\theta-\theta^*)^\top A_{\pi_{\theta}} (\theta-\theta^*) + \lVert \theta-\theta^*\rVert \lVert A_{\pi_{\theta}}-A_{\pi_{\theta^*}} \rVert \lVert\theta^*\rVert + \lVert \theta-\theta^*\rVert\lVert b_{\pi_{\theta}}-b_{\pi_{\theta^*}} \rVert.  
\end{align*}
By Lemma \ref{lem:Lipschitz_A_b} and Assumption \ref{Assumption:policy} we obtain
\begin{align*}
(\theta-\theta^*)^\top h(\theta) &\leq 
(\theta-\theta^*)^\top A_{\pi_{\theta}} (\theta-\theta^*) + C_1\lVert \theta-\theta^*\rVert^2 \lVert\theta^*\rVert + C_2\lVert \theta-\theta^*\rVert^2\\
& = (\theta-\theta^*)^\top (A_{\pi_{\theta}}+(C_1 \lVert\theta^*\rVert + C_2)I)(\theta-\theta^*). 
\end{align*}
By Lemma \ref{lem:neg_def} $A_{\pi_{\theta}}$ is negative definite. Hence, for $C_1$ and $C_2$ sufficiently small, $A_{\pi_{\theta}}+(C_1 \lVert\theta^*\rVert + C_2)I$ is negative definite, i.e.\ the stability condition \eqref{eq:stability2} is satisfied. 
\\

Hence, since the Robbins-Monro assumption \eqref{eq:robbins}, the square integrability condition \eqref{eq:moment2}, and the stability condition \eqref{eq:stability2} are satisfied, Theorem \ref{thm:convergence} follows.

\section{Discussion}\label{sec:Discussion}
We have shown that if the behaviour policy is $\varepsilon$-soft and Lipschitz continuous w.r.t.\ the weight vector, with small enough Lipschitz constant, then SARSA with linear function approximation will converge with probability one when considering a random horizon MDP. This is in line with earlier convergence results for infinite horizon discounted MDPs in \cite{perkins2003convergent,melo2008analysis}. 

For the variant of SARSA considered here, the weight vector and the behaviour policy are only updated at the end of each trajectory, not after each iteration. This variant of the algorithm should work well if the trajectories are not too long, but could cause slow convergence in practice in the case of very long trajectories. However, for a random horizon MDP with very long trajectories, it is possible that earlier results for infinite horizon MDPs will hold if replacing the stationary distribution of the Markov chain induced by a policy $\pi$ with a quasi-stationary distribution, at least for the discounted version of the problem. Obtaining convergence results for the online version of the algorithm, where the weight vector and policy are updated after each iteration, could still be of interest for problems with trajectories of medium length, i.e.\ too short for the existence of a quasi-stationary distribution, but long enough to cause slow convergence in practice. 
It is possible that Theorem 17, p.~239, in Benveniste et al.\ \cite{benveniste2012adaptive} could be used to prove convergence in this case, similarly to what is done in \cite{melo2008analysis}, but it would be more complex to ascertain if the various assumptions in this theorem are satisfied, since the Markov chain induced by a policy is not ergodic in the random horizon case. 

Furthermore, the theorem in this paper suffers from the same limitations as 
the theorems obtained in the infinite horizon discounted case, 
discussed in Perkins \& Precup \cite{perkins2003convergent} and Melo et al.\ \cite{melo2008analysis}. As described in \cite{perkins2003convergent}, the value of $C_0$ in Theorem 1 is not specified in the theorem, and depends on the properties of the MDP, which might be unknown (e.g.\ transition probabilities). Moreover, there is no guarantee on how close the approximation is to the true optimal action-value function and the true optimal policy. As discussed in \cite{melo2008analysis}, to approximate the true optimal action-value function and policy well, the behaviour policy over time needs to approach the greedy policy, by e.g.\ having a decaying exploration rate. This would, however, lead to an increased Lipschitz constant (since the greedy policy is discontinuous), hence the condition that the Lipschitz constant is sufficiently small might no longer hold.

\bibliographystyle{plain}
\bibliography{RL_conv_v1_nocomments}

\appendix

\section{Norms and norm inequalities} \label{app:norms}
The following norm definitions are used:
\begin{itemize}
\item For any vector $x\in\mathbb{R}^n$ 
\begin{align*}
&\lVert x\rVert=\sqrt{\sum_{i=1}^n x_i^2} \quad \text{(Euclidean norm)},\\
&\lVert x\rVert_1 = \sum_{i=1}^n|x_i|,\\
&\lVert x\rVert_{\infty} = \max_i \lvert x_i\rvert \quad \text{(infinity norm)}.
\end{align*} 
\item For any matrix $A\in\mathbb{R}^{m\times n}$
\begin{align*}
&\lVert A\rVert =
\sqrt{\lambda_{\max}(A^\top A)} \quad \text{(spectral norm)},\\
&\lVert A\rVert_{\infty} =\max_{i}\sum_{j=1}^n \lvert A_{i,j} \rvert  \quad  \text{(maximum absolute row sum norm)},\\
&\lVert A\rVert_{\max} =\max_{i,j}\lvert A_{i,j}\rvert \quad \text{(max norm)}. 
\end{align*}
\end{itemize}

We also use the following well known norm equivalences:
\begin{itemize}
\item For any vector $x\in\mathbb{R}^n$ 
\begin{align*}
&\lVert x\rVert_{\infty}\leq \lVert x\rVert \leq \lVert x\rVert_1.
\end{align*}
\item For any matrix $A\in\mathbb{R}^{m\times n}$
\begin{align*}
&\lVert A\rVert\leq \sqrt{mn}\lVert A\rVert_{\max},\\
&\lVert A\rVert\leq \sqrt{m}\lVert A\rVert_{\infty}.
\end{align*}
\end{itemize}

\section{Proofs of Lemmas \ref{lem:neg_def}, \ref{lem:fixed_point_F}, \ref{lem:F_quasi_contraction}-\ref{lem:F_pi_theta_fixed_point}, and \ref{lem:P_pi}} \label{app:profs}

\subsection{Proof of Lemma \ref{lem:neg_def}}
The proof of Lemma \ref{lem:neg_def} is identical to the last part of the proof of Lemma 6.10 in Bertsekas \& Tsitsiklis \cite{bertsekas1996neuro}, but here considering a Markov chain over state-action pairs, with $P_{\pi}$ ($P$ in \cite{bertsekas1996neuro}) and $D_{\pi}$ ($Q$ in \cite{bertsekas1996neuro}) defined accordingly, and the policy $\pi$ being $\varepsilon$-soft. 
\begin{proof}
For any $\varepsilon$-soft policy $\pi$, similarly to what is shown in the proof of Lemma \ref{lem:theta_pi}, but without requiring Assumption \ref{Assumption:contraction},
\begin{align*}
\lVert P_{\pi}q\rVert_{\eta_{\pi}}^2 &\leq 
\sum_{s',a'}q(s',a')^2(\eta_{\pi}(s',a')-\lambda_{\pi}(s',a'))
\leq\sum_{s',a'}q(s',a')^2\eta_{\pi}(s',a') = \lVert q\rVert_{\eta_{\pi}}^2, 
\end{align*}
for all $q\in\mathbb{R}^{|\mathcal S||\mathcal A|}$. Hence
\begin{align*}
q^\top D_{\pi} P_{\pi} q \leq \lVert q\rVert_{\eta_{\pi}}\lVert P_{\pi} q\rVert_{\eta_{\pi}}\leq \lVert q\rVert_{\eta_{\pi}}^2 = q^\top D_{\pi} q. 
\end{align*}
For this inequality to be an equality, we need $q$ and $P_{\pi}q$ to be colinear, and $\lVert P_{\pi} q\rVert_{\eta_{\pi}}=\lVert q\rVert_{\eta_{\pi}}$. Thus, the inequality is strict unless $P_{\pi} q= q$ or $P_{\pi} q= -q$, which means that $P_{\pi}^{2m}q = q$ for all $m\geq 0$. Since the policy is proper, $P_{\pi}^{2m}$ converges to zero, hence we must have $q=0$, i.e.\ the inequality is strict for all $q\neq 0$. Hence, the matrix $D_{\pi}(P_{\pi}-I)$ is negative definite. Using Assumption \ref{Assumption:Phi}(i) (linear independence of columns of $\Phi$) $A_{\pi}$ is negative definite. 
\end{proof}

\subsection{Proof of Lemma \ref{lem:fixed_point_F}}
The proof of Lemma \ref{lem:fixed_point_F} is identical to the proof of Lemma 5.3 in de Farias \& Van Roy \cite{de2000existence}, but using our definitions of $s_{\pi_{\theta}}$, $F_{\pi_{\theta}}^{\alpha}$, $T_{\pi_{\theta}}$ and $D_{\pi_{\theta}}$ ($s_{\delta}$, $F_{\delta}^{\gamma}$, $T_{\delta}$, and $D_{\mu_r^{\delta}}$ in \cite{de2000existence}), and that the policy $\pi$ is $\varepsilon$-soft. We also use that $D_{\pi}$ is positive definite for any $\varepsilon$-soft policy $\pi$ under Assumption \ref{Assumption:states}.

\begin{proof}
Let $\theta$ be a fixed point of $F_{\pi}^{\alpha}$. Then $s_{\pi}(\theta)=0$, hence
\begin{align*}
\Phi^\top D_{\pi}\Phi\theta&=\Phi^\top D_{\pi}(r+P_{\pi}\Phi\theta), \\
\Phi(\Phi^\top D_{\pi}\Phi)^{-1}\Phi^\top D_{\pi}\Phi\theta&=\Phi(\Phi^\top D_{\pi}\Phi)^{-1}\Phi^\top D_{\pi}(r+P_{\pi}\Phi\theta), \\
\Phi\theta &= \Pi_{\pi}T_{\pi}\Phi\theta,
\end{align*}
and $\Phi\theta$ is a fixed point of $H_{\pi}$.
The reverse can be shown by reversing the steps, and an entirely analogous argument can be used to show that $\theta$ is a fixed point of $F_{\pi_{\theta}}^{\alpha}$ if and only if $\Phi\theta$ is a fixed point of $H_{\pi_{\theta}}$. 
\end{proof}

\subsection{Proof of Lemma \ref{lem:F_quasi_contraction}}
The proof of Lemma \ref{lem:F_quasi_contraction} is essentially identical to the proof of Lemma 5.4 in de Farias \& Van Roy \cite{de2000existence}, but using our definition of $H_{\pi}$ and using a $\varepsilon$-soft policy $\pi$. 

\begin{proof}
We have
\begin{align*}
\lVert F_{\pi}^{\alpha}(\theta)-\theta_{\pi}\rVert^2&=\lVert \theta+\alpha s_{\pi}(\theta)-\theta_{\pi}\rVert^2
= \lVert \theta-\theta_{\pi}\rVert^2 + 2\alpha(\theta-\theta_{\pi})^\top s_{\pi}(\theta) + \alpha^2\lVert s_{\pi}(\theta)\rVert^2. 
\end{align*}
For the second term, note that (using Lemma \ref{lem:theta_pi} and that $H_{\pi}$ is a contraction) for all $\varepsilon$-soft $\pi$, there exists $\beta\in[0,1)$ such that
\begin{align} \label{eq:H_pi}
\lVert H_{\pi}\Phi\theta-\Phi \theta_{\pi}\rVert_{\eta_{\pi}} = \lVert H_{\pi}\Phi\theta-H_{\pi}\Phi \theta_{\pi}\rVert_{\eta_{\pi}}\leq \beta\lVert\Phi\theta-\Phi\theta_{\pi}\rVert_{\eta_{\pi}},
\end{align}
and furthermore that
\begin{align*}
(\theta-\theta_{\pi})^\top s_{\pi}(\theta) &= (\theta-\theta_{\pi})^\top\Phi^\top D_{\pi}(T_{\pi}\Phi\theta - \Phi\theta) \\
&= (\theta-\theta_{\pi})^\top \Phi^\top D_{\pi}\Phi (\Phi^\top D_{\pi}\Phi)^{-1}\Phi^\top D_{\pi}(T_{\pi}\Phi\theta-\Phi\theta)\\
&= (\Phi\theta-\Phi\theta_{\pi})^\top D_{\pi} (\Phi (\Phi^\top D_{\pi}\Phi)^{-1}\Phi^\top D_{\pi}T_{\pi}\Phi\theta -\Phi\theta) \\
&= (\Phi\theta-\Phi\theta_{\pi})^\top D_{\pi} (H_{\pi}\Phi\theta - \Phi\theta) = \langle \Phi\theta-\Phi\theta_{\pi}, H_{\pi}\Phi\theta - \Phi\theta\rangle_{\eta_{\pi}},
\end{align*}
where $\langle \cdot,\cdot\rangle_{\eta_{\pi}}$ denotes the weighted inner product, i.e.\ $\langle x,y\rangle_{\eta_{\pi}}=x^\top D_{\pi} y$. Now, using \eqref{eq:H_pi},
\begin{align*}
\langle \Phi\theta-\Phi\theta_{\pi}, H_{\pi}\Phi\theta - \Phi\theta\rangle_{\eta_{\pi}} &= \langle \Phi\theta-\Phi\theta_{\pi}, (H_{\pi}\Phi\theta -\Phi\theta_{\pi}) -(\Phi\theta_{\pi} - \Phi\theta) \rangle_{\eta_{\pi}}\\
&=  \langle \Phi\theta-\Phi\theta_{\pi}, H_{\pi}\Phi\theta -\Phi\theta_{\pi}\rangle_{\eta_{\pi}} -\lVert \Phi\theta - \Phi\theta_{\pi}\rVert_{\eta_{\pi}}^2\\
&\leq \lVert \Phi\theta - \Phi\theta_{\pi}\rVert_{\eta_{\pi}} \lVert H_{\pi}\Phi\theta -\Phi\theta_{\pi} \rVert_{\eta_{\pi}} - \lVert \Phi\theta - \Phi\theta_{\pi}\rVert_{\eta_{\pi}}^2\\
&\leq (\beta-1)\lVert \Phi\theta - \Phi\theta_{\pi}\rVert_{\eta_{\pi}}^2 = (\beta-1)(\theta-\theta_{\pi})^\top \Phi^\top D_{\pi}\Phi(\theta-\theta_{\pi}). 
\end{align*}
Hence, 
\begin{align*}
(\theta-\theta_{\pi})^\top s_{\pi}(\theta) \leq (\beta-1)\lVert \theta-\theta_{\pi}\rVert\lVert  \Phi^\top D_{\pi}\Phi(\theta-\theta_{\pi})\rVert \leq (\beta-1)\lVert  \Phi^\top D_{\pi}\Phi\rVert\lVert \theta-\theta_{\pi}\rVert^2. 
\end{align*}
Since $D_{\pi}$ is a positive definite matrix for all $\varepsilon$-soft policies $\pi$, $\Phi^\top D_{\pi}\Phi$ is positive definite and symmetric, hence $\lVert  \Phi^\top D_{\pi}\Phi\rVert>0$. It follows that there exists a constant $C_1>0$ such that
\begin{align*}
(\theta-\theta_{\pi})^\top s_{\pi}(\theta)\leq -C_1\lVert \theta -\theta_{\pi}\rVert^2, 
\end{align*}
namely $C_1 = (1-\beta)\max_{\pi\in\Delta_{\varepsilon}}\lVert \Phi^\top D_{\pi}\Phi\rVert$, where the maximum is attained since (by Lemma \ref{lem:eta_cont_pi}) $D_{\pi}$ is a continuous function of $\pi$, and the set of all $\varepsilon$-soft policies is compact. 

Note that $\Phi^\top D_{\pi}\Pi_{\pi} = \Phi^\top D_{\pi}\Phi (\Phi^\top D_{\pi}\Phi)^{-1}\Phi^\top D_{\pi}=\Phi^\top D_{\pi}$. Furthermore, let $\phi_i$ be the $i$th column of $\Phi$.
Then
\begin{align*}
\lVert s_{\pi}(\theta)\rVert^2 &= \lVert \Phi^\top D_{\pi}(T_{\pi}\Phi\theta-\Phi\theta)\rVert^2 = \sum_{i=1}^d\big(\phi_i^\top D_{\pi}(T_{\pi}\Phi\theta-\Phi\theta)\big)^2\\
&=  \sum_{i=1}^d\big(\phi_i^\top D_{\pi}(\Pi_{\pi}T_{\pi}\Phi\theta-\Phi\theta)\big)^2
\leq \sum_{i=1}^d \lVert \phi_i\rVert_{\eta_{\pi}}^2\lVert \Pi_{\pi}T_{\pi}\Phi\theta- \Phi \theta\rVert_{\eta_{\pi}}^2\\
&\leq \sum_{i=1}^d \lVert \phi_i\rVert_{\eta_{\pi}}^2(\lVert \Pi_{\pi}T_{\pi}\Phi\theta- \Phi \theta_{\pi}\rVert_{\eta_{\pi}} + \lVert \Phi\theta_{\pi}- \Phi \theta\rVert_{\eta_{\pi}})^2\\
&\leq \sum_{i=1}^d \lVert \phi_i\rVert_{\eta_{\pi}}^2(\beta \lVert \Phi\theta-\Phi\theta_{\pi}\rVert_{\eta_{\pi}}+ \lVert \Phi\theta_{\pi}- \Phi \theta\rVert_{\eta_{\pi}})^2\\
&=(\beta+1)^2\sum_{i=1}^d \lVert \phi_i\rVert_{\eta_{\pi}}^2\lVert \Phi\theta_{\pi}-\Phi\theta\rVert_{\eta_{\pi}}^2, 
\end{align*}
and (similarly to above) it follows that there exists a constant $C_2>0$ (independent of $\pi$) such that $\lVert s_{\pi}(\theta)\rVert^2\leq C_2\lVert \theta-\theta_{\pi}\rVert^2$.
Hence
\begin{align*}
\lVert F_{\pi}^{\alpha}(\theta)-\theta_{\pi}\rVert^2&\leq
 \lVert \theta-\theta_{\pi}\rVert^2 + 2\alpha(\theta-\theta_{\pi})^\top s_{\pi}(\theta) + \alpha^2\lVert s_{\pi}(\theta)\rVert^2\\
 &\leq (1-2\alpha C_1+\alpha^2C_2)\lVert \theta-\theta_{\pi}\rVert^2.
\end{align*}
Thus, with $\alpha^*=2C_1/C_2$ (independent of $\pi$) and $\alpha\in(0,\alpha^*)$, we see that 
\begin{align*}
1-2\alpha C_1+\alpha^2C_2=1+C_2 \alpha(\alpha-\alpha^*)<1, 
\end{align*}
i.e.\ there exists $\beta_{\alpha}\in(0,1)$ such that
\begin{align*}
\lVert F_{\pi}^{\alpha}(\theta)-\theta_{\pi}\rVert^2\leq \beta_{\alpha}\lVert \theta-\theta_{\pi}\rVert^2.
\end{align*}
\end{proof}

\subsection{Proof of Lemma \ref{lem:F_pi_theta_fixed_point}}
The proof of Lemma \ref{lem:F_pi_theta_fixed_point} is essentially identical to the proof of Lemma 5.5 in de Farias \& Van Roy \cite{de2000existence}, but using our definitions of $F_{\pi_{\theta}}^{\alpha}$ ($F_{\delta}^{\gamma}$ in \cite{de2000existence}). Furthermore, \cite{de2000existence} use that the set of all stochastic policies is compact, we instead use that the set of $\varepsilon$-soft policies is compact. Moreover, in our case it is the Lipschitz continuity of $\pi_{\theta}$ that implies that $F_{\pi_{\theta}}^{\alpha}$ is continuous in $\theta$, rather than the specific choice of behaviour policy (softmax policy) considered in \cite{de2000existence}.

\begin{proof}
By Lemma \ref{lem:theta_pi_cont}, $\theta_{\pi}$ is a continuous function of $\pi$. Since the set of $\varepsilon$-soft policies, $\Delta_{\varepsilon}$, is compact, 
the set $\Theta = \{ \theta_{\pi}:\pi\in\Delta_{\varepsilon}\}$ is also compact. Let $\Theta_{\max} = \max\{\lVert \theta \rVert:\theta\in\Theta\}$. 

Note that if we establish that a fixed point exists for some $\alpha>0$, then by Lemma \ref{lem:fixed_point_F} this fixed point is also a fixed point for all other $\alpha>0$. Using Lemma \ref{lem:F_quasi_contraction}, we can choose $\alpha>0$ such that there is a $\beta\in(0,1)$ with
\begin{align*}
\lVert F_{\pi}^\alpha(\theta)-\theta_{\pi}\rVert\leq\beta\lVert\theta -\theta_{\pi}\rVert,
\end{align*}
for all $\varepsilon$-soft $\pi$. Then
\begin{align*}
\lVert F_{\pi_{\theta}}^\alpha(\theta)\rVert\leq \lVert F_{\pi_{\theta}}^\alpha(\theta)-\theta_{\pi_{\theta}}\rVert + \lVert \theta_{\pi_{\theta}}\rVert 
\leq \beta\lVert\theta -\theta_{\pi_{\theta}}\rVert +\Theta_{\max}\leq \beta\lVert \theta\rVert+(\beta+1)\Theta_{\max}. 
\end{align*}
Hence, the set $\bar\Theta = \{\theta:\lVert\theta\rVert\leq (1+\beta)\Theta_{\max}/(1-\beta)\}$ is closed under $F_{\pi_{\theta}}^{\alpha}$, since, if $\theta\in\bar\Theta$, then by the above
\begin{align*}
\lVert F_{\pi_{\theta}}^{\alpha}(\theta)\rVert\leq\beta\lVert \theta\rVert+(\beta+1)\Theta_{\max}\leq\beta\frac{1+\beta}{1-\beta}\Theta_{\max}+(1+\beta)\Theta_{\max}=\frac{1+\beta}{1-\beta}\Theta_{\max},
\end{align*}
i.e.\ $F_{\pi_{\theta}}^{\alpha}(\theta)\in\bar\Theta$. Using this, and that $F_{\pi_{\theta}}^{\alpha}$ is a continuous function of $\theta$ (since $\pi_{\theta}$ is Lipschitz continuous w.r.t.\ $\theta$, $P_{\pi}$ is a continuous function of $\pi$, $D_{\pi}$ is a continuous function of $\pi$ by Lemma \ref{lem:eta_cont_pi}, hence $T_{\pi_{\theta}}$ is a continuous function of $\theta$, which implies that $F_{\pi_{\theta}}^{\alpha}$ is a continuous function of $\theta$), by the Brouwer fixed point theorem $F_{\pi_{\theta}}^{\alpha}$ possesses a fixed point. 
\end{proof}

\subsection{Proof of Lemma \ref{lem:P_pi}}
The proof of Lemma \ref{lem:P_pi} is essentially identical to the proof of Lemma 1 in Perkins \& Precup \cite{perkins2003convergent}, but using our definition of $P_{\pi}$, and correcting what appears to be a minor error in the proof, namely that $\lVert A \rVert \leq n \lVert A\rVert_{\max}$ for $A\in\mathbb{R}^{n\times n}$, rather than $\lVert A \rVert \leq \sqrt{n} \lVert A\rVert_{\max}$. 

\begin{proof}
Let $\pi_1$ and $\pi_2$ be fixed, and let $i$ and $j$ correspond to the $(s,a)$th row and $(s',a')$th column of $P_{\pi}$, respectively. Then
\begin{align*}
\lVert P_{\pi_1}-P_{\pi_2}\rVert&\leq \sqrt{(|\mathcal S||\mathcal A|)^2}\lVert P_{\pi_1}-P_{\pi_2}\rVert_{\max} = |\mathcal S||\mathcal A|\max_{i,j}\lvert (P_{\pi_1})_{i,j}-(P_{\pi_2})_{i,j}\rvert\\
&=|\mathcal S||\mathcal A|\max_{s,a,s',a'}\lvert p(s'|s,a)(\pi_{1}(a'|s') - \pi_{2}(a'|s') )\rvert \leq |\mathcal S||\mathcal A|\max_{s',a'}\lvert \pi_{1}(a'|s') - \pi_{2}(a'|s') \rvert\\
&=|\mathcal S||\mathcal A|\lVert \pi_1-\pi_2\rVert_{\infty}\leq |\mathcal S||\mathcal A| \lVert \pi_1-\pi_2\rVert,
\end{align*}
i.e.\ $C_P=|\mathcal S||\mathcal A|$.
\end{proof}

\section{Proof that $h(\theta)=b_{\pi_{\theta}} + A_{\pi_{\theta}}\theta$} \label{app:h}

Using the convention $\phi_u^{(t)}=0$ for $u\geq T^{(t)}$,
\begin{align*}
\E_{\theta}\bigg[\sum_{u=0}^{T^{(t+1)}-1}\phi_u^{(t+1)}(\phi_{u}^{(t+1)})^\top \bigg]
&=\E_{\theta}\bigg[\sum_{u=0}^{\infty}\phi(S_u^{(t+1)},A_u^{(t+1)})\phi(S_u^{(t+1)},A_u^{(t+1)})^\top \bigg]\\
&=\sum_{u=0}^{\infty}\E_{\theta} [\phi(S_u^{(t+1)},A_u^{(t+1)})\phi(S_u^{(t+1)},A_u^{(t+1)})^\top ]\\
&=\sum_{u=0}^{\infty}\sum_{s\in\mathcal{S},a\in\mathcal{A}} P_{\pi_{\theta}}(S_u^{(t+1)}=s,A_u^{(t+1)}=a)\phi(s,a)\phi(s,a)^\top\\
&=\sum_{s\in\mathcal{S},a\in\mathcal{A}}\eta_{\pi_{\theta}}(s,a)\phi(s,a)\phi(s,a)^\top = \Phi^\top D_{\pi_{\theta}} \Phi,
\end{align*}
\begin{align*}
\E_{\theta}&\bigg[\sum_{u=0}^{T^{(t+1)}-1}\phi_u^{(t+1)}(\phi_{u+1}^{(t+1)})^\top \bigg]
=\E_{\theta}\bigg[\sum_{u=0}^{\infty}\phi(S_u^{(t+1)},A_u^{(t+1)})\phi(S_{u+1}^{(t+1)},A_{u+1}^{(t+1)})^\top \bigg]\\
&=\sum_{u=0}^{\infty}\E_{\theta} [\phi(S_u^{(t+1)},A_u^{(t+1)})\phi(S_{u+1}^{(t+1)},A_{u+1}^{(t+1)})^\top ]\\
 &=\sum_{u=0}^{\infty}\sum_{s,s'\in\mathcal{S},a,a'\in\mathcal{A}} P_{\pi_{\theta}}(S_u^{(t+1)}=s,A_u^{(t+1)}=a)p_{\pi_{\theta}}(s',a'\mid s,a)\phi(s,a)\phi(s',a')^\top\\
 & =\sum_{s,s'\in\mathcal{S},a,a'\in\mathcal{A}} \eta_{\pi_{\theta}}(s,a)p_{\pi_{\theta}}(s',a'\mid s,a)\phi(s,a)\phi(s',a')^\top = \Phi^\top D_{\pi_{\theta}} P_{\pi_{\theta}} \Phi,
\end{align*}
and 
\begin{align*}
\E_{\theta}&\bigg[\sum_{u=0}^{T^{(t+1)}-1}\phi_u^{(t+1)} r_u^{(t+1)} \bigg]
= \E_{\theta}\bigg[\sum_{u=0}^{\infty}\phi(S_u^{(t+1)},A_u^{(t+1)}) r(S_{u}^{(t+1)},A_{u}^{(t+1)},S_{u+1}^{(t+1)})^\top \bigg]\\
&=\sum_{u=0}^{\infty}\sum_{s\in\mathcal{S},a\in\mathcal{A},s'\in\mathcal{S}^+} P_{\pi_{\theta}}(S_u^{(t+1)}=s,A_u^{(t+1)}=a) p(s'\mid s,a)\phi(s,a) r(s,a,s')\\
& = \sum_{s\in\mathcal{S},a\in\mathcal{A}}\sum_{u=0}^\infty P_{\pi_{\theta}}(S_u^{(t+1)}=s,A_u^{(t+1)}=a)\phi(s,a)\sum_{s'\in\mathcal{S}^+} p(s'\mid s,a)r(s,a,s')\\
&= \sum_{s\in\mathcal{S},a\in\mathcal{A}} \eta_{\pi_{\theta}}(s,a)\phi(s,a) r(s,a) = \Phi^\top D_{\pi_{\theta}} r.
\end{align*}
Thus
\begin{align*}
h(\theta)=\Phi^\top D_{\pi_{\theta}} r + \Phi^\top D_{\pi_{\theta}}(P_{\pi_{\theta}}-I)\Phi\theta = b_{\pi_{\theta}} + A_{\pi_{\theta}}\theta. 
\end{align*}

\end{document}